\documentclass[12pt, a4paper]{article}

\usepackage[a4paper]{geometry}
\usepackage{pgfplots}
\pgfplotsset{
    x tick style={color=black},
    y tick style={color=black}
}

\usepackage{url}
\usepackage{listings}
\usepackage{amssymb}
\usepackage{graphicx}
\usepackage{algcompatible}
\usepackage{algorithm}
\usepackage{url}
\usepackage{rotating}
\usepackage{booktabs}
\usepackage{subfig}
\usepackage{amsmath}
\usepackage{bbm}

\renewcommand{\labelenumi}{(\alph{enumi})}
\renewcommand\theenumi\labelenumi

\usepackage[english]{babel}

\usepackage[utf8]{inputenc}
\usepackage{xspace}
\usepackage{amsmath,amsthm,amssymb,mathtools}
\usepackage{lmodern}

\usepackage[algo2e,ruled,vlined,linesnumbered]{algorithm2e}

\usepackage{xcolor}
\usepackage{tikz}
\usepackage{graphicx}
\usepackage{soul} 

\allowdisplaybreaks[4]
\newtheorem{theorem}{Theorem}
\newtheorem*{theorem*}{Theorem}
\newtheorem{lemma}[theorem]{Lemma}
\newtheorem*{lemma*}{Lemma}
\newtheorem{corollary}[theorem]{Corollary}
\newtheorem{definition}[theorem]{Definition}

\newcommand{\oea}{$(1 + 1)$~EA\xspace}

\newcommand{\NSGA}{\mbox{NSGA-II}\xspace}
\newcommand{\NSGAthree}{\mbox{NSGA-III}\xspace}
\newcommand{\SMS}{\mbox{SMS-EMOA}\xspace}

\newcommand{\om}{\textsc{OneMax}\xspace}
\newcommand{\jump}{\textsc{Jump}\xspace}

\newcommand{\omm}{\textsc{OneMinMax}\xspace}
\newcommand{\momm}{$m$\textsc{OneMinMax}\xspace}
\newcommand{\mojzj}{$m$\textsc{OJZJ}\xspace}
\newcommand{\onemax}{\om}
\newcommand{\lo}{\textsc{LeadingOnes}\xspace}
\newcommand{\lotz}{\textsc{LOTZ}\xspace}

\newcommand{\ojzj}{\textsc{OJZJ}\xspace}
\newcommand{\cocz}{\textsc{CountingOnesCountingZeroes}\xspace}

\newcommand{\DLTB}{\textsc{DLTB}\xspace}

\newcommand{\R}{\ensuremath{\mathbb{R}}}

\newcommand{\N}{\ensuremath{\mathbb{N}}} 

\DeclareMathOperator{\HV}{HV}

\let\originalleft\left
\let\originalright\right
\renewcommand{\left}{\mathopen{}\mathclose\bgroup\originalleft}
\renewcommand{\right}{\aftergroup\egroup\originalright}

\usepackage{hyperref}

\begin{document}
{\sloppy
\date{}
\title{Runtime Analysis of the \SMS for Many-Objective Optimization}
\author{Weijie Zheng\\
         School of Computer Science and Technology\\
         International Research Institute for Artificial Intelligence\\
       Harbin Institute of Technology\\
        Shenzhen, China
\and Benjamin Doerr\thanks{Corresponding author.}\\ Laboratoire d'Informatique (LIX)\\ \'Ecole Polytechnique, CNRS\\ Institute Polytechnique de Paris\\ Palaiseau, France
}

\maketitle

\begin{abstract}
The widely used multi-objective optimizer NSGA-II was recently proven to have considerable difficulties in many-objective optimization. In contrast, experimental results in the literature show a good performance of the SMS-EMOA, which can be seen as a steady-state NSGA-II that uses the hypervolume contribution instead of the crowding distance as the second selection criterion. 

This paper conducts the first rigorous runtime analysis of the SMS-EMOA for many-objective optimization. To this aim, we first propose a many-objective counterpart of the bi-objective \ojzj benchmark. We prove that SMS-EMOA computes the full Pareto front of this benchmark in an expected number of $O(\mu M n^k)$ iterations, where $n$ denotes the problem size (length of the bit-string representation), $k$ the gap size (a difficulty parameter of the problem), $M=(2n/m-2k+3)^{m/2}$ the size of the Pareto front, and $\mu$ the population size (at least the same size as the largest incomparable set). This result together with the existing negative result for the original NSGA-II shows that, in principle, the general approach of the NSGA-II is suitable for many-objective optimization, but the crowding distance as tie-breaker has deficiencies.

We obtain three additional insights on the SMS-EMOA. Different from a recent result for the bi-objective \ojzj benchmark, a recently proposed stochastic population update often does not help for its many-objective counterpart. It at most results in a speed-up by a factor of order $2^{k} / \mu$, which is $\Theta(1)$ for large $m$, such as $m>k$. On the positive side, we prove that heavy-tailed mutation irrespective of the number $m$ of objectives results in a speed-up of order $k^{0.5+k-\beta}/e^k$.
Finally, we conduct the first runtime analyses of the SMS-EMOA on the classic bi-objective \omm and \lotz benchmarks and show that the SMS-EMOA has a performance comparable to the GSEMO and the NSGA-II.

Our main technical insight, a general condition ensuring that the SMS-EMOA does not lose Pareto-optimal objective values, promises to be useful also in other runtime analyses of this algorithm.
\end{abstract}

\section{Introduction}\label{sec:intro}

The \NSGA \cite{DebPAM02} is the most widely-applied multi-objective evolutionary algorithm (MOEA). 
Non-dominated sorting and crowding distance are its two main building blocks differentiating it from  basic MOEAs such as the SEMO \cite{LaumannsTZ04}, the GSEMO \cite{Giel03} or the $(\mu+1)$ SIBEA~\cite{NguyenSN15}. 
Zheng, Liu, and Doerr \cite{ZhengLD22} conducted the first runtime analysis of the \NSGA (see \cite{ZhengD23aij} for the journal version). 
This work quickly inspired many interesting follow-up results in bi-objective optimization~\cite{BianQ22,DoerrQ23tec,DoerrQ23LB,DoerrQ23crossover,DangOSS23gecco,CerfDHKW23,ZhengD24approx,DangOS24}. 

In contrast to these positive results for two objectives, Zheng and Doerr~\cite{ZhengD24many} proved that for $m \ge 3 $ objectives the \NSGA needs at least exponential time (in expectation and with high probability) to cover the full Pareto front of the $m$-objective \omm benchmark, a simple many-objective version of the basic \onemax problem. 
They stated that the main reason for this low efficiency is the independent computation of the crowding distance in each objective, which suggests that the \NSGA has similar difficulties on other many-objective problems (this was very recently confirmed for the many-objective \lotz problem~\cite{DoerrKK24arxiv}). 

The difficulties of the \NSGA here could be overcome by switching to the \NSGAthree, which provably can efficiently solve several many-objective benchmarks including \omm and \lotz~\cite{WiethegerD23,OprisDNS24}, however, at the price of correctly setting a complex system of parameters, namely the system of reference points the \NSGAthree builds on. 

With the \SMS, a second interesting variant of the \NSGA was proposed by Beume, Naujoks, and Emmerich~\cite{BeumeNE07}, and one which has the population size~$\mu$ as single parameter. This algorithm is a steady-state variant of the \NSGA (that is, in each iteration only a single offspring is generated and possibly integrated into the population) that further replaces the crowding distance as secondary selection criterion with the classic hypervolume contribution. Many empirical works (see the more than 2,000 papers citing \cite{BeumeNE07}) confirmed the good performance of the \SMS in many-objective optimization. The first mathematical runtime analysis of the \SMS was conducted very recently by Bian et al.~\cite{BianZLQ23}, who proved that its expected runtime on the bi-objective \ojzj problem is $O(n^{k+1})$. They also proposed a stochastic population update mechanism and proved that it has the often superior runtime of $O(n^{k+1}\min\{1,n/2^{k/4}\})$. 

\emph{Our Contributions:} This paper conducts the first mathematical runtime analysis of the \SMS for more than two objectives. To this aim, we first define the \mojzj benchmark, an $m$-objective counterpart of the bi-objective \ojzj problem~\cite{ZhengD23ecj}, the problem analyzed in the first runtime analysis for the \SMS \cite{BianZLQ23}. 
We note that the \mojzj problem is the first multimodal many-objective benchmark proposed for a theoretical analysis, to the best of our knowledge. 

In our mathematical runtime analysis of the \SMS on this benchmark (Theorem~\ref{thm:sms}), we prove that this algorithm can compute the full Pareto front in an expected number of $O(\mu M n^k)$ iterations, where $n$ is the problem size, $k$ the gap size (a difficulty parameter of the problem that is usually small), $m$ the number of objectives,  $M=(2n/m-2k+3)^{m/2}$ the size of the Pareto front, and where the population size $\mu$ has to be at least as large as the largest set of incomparable solutions (which is not larger than $(2n/m+1)^{m/2}$), an assumption made in most mathematical runtime analyses of MOEAs.

We recall that the original \NSGA needs at least exponential time to optimize the \momm problem, which is a special case of \mojzj with gap size $k=1$. Since the \SMS employs non-dominated sorting, but replaces the crowding distance in the original \NSGA by the hypervolume contribution, our result in a similar fashion as the analysis of the \NSGAthree in \cite{WiethegerD23,OprisDNS24} suggests that the general approach of the \NSGA is suitable for many-objective optimization and that it is only the crowding distance as tie-breaker which is not appropriate for more than two objectives. 

We then analyze whether the better performance of the \SMS on the bi-objective \ojzj problem achieved via a new stochastic population update \cite{BianZLQ23} extends to the $m$-objective \mojzj problem. Unfortunately, we shall observe that only a speed-up of order at most $\max\{1,2^{k} / \mu\}$ is obtained, which is $1$  when $m$ is mildly large,  simply because $\mu \ge M = (2n/m - 2k + 3)^{m/2}$. 

On the positive side, we show that the advantage of heavy-tailed mutation is preserved. We analyze the \SMS with heavy-tailed mutation on the \mojzj benchmark and prove that a speed-up of order $k^{0.5+k-\beta}/e^k$ is achieved. This is the same speed-up as observed previously for single-objective and bi-objective \jump problems~\cite{DoerrLMN17,ZhengD23ecj,DoerrQ23tec}. We note that this is the first theoretical work to support the usefulness of heavy-tailed mutation in many-objective optimization.

Finally, since so far the performance of the \SMS was only analyzed on the bi-objective \ojzj problem~\cite{BianZLQ23} (and, in a work~\cite{ZhengLDD24} parallel to this, on the not overly prominent DLTB benchmark), we conduct mathematical runtime analyses of the \SMS also on the two most prominent bi-objective benchmarks \omm and \lotz. We prove that the \SMS finds the Pareto fronts of these benchmarks in an expected number of at most $2e(n+1)n(\ln n+1)$ iterations for \omm and at most $2en^2(n+1)$ for \lotz. These are the same asymptotic runtimes (in terms of fitness evaluations) as known for the GSEMO and the \NSGA. 

The technical key to all these results are two simple, but general results ensuring that the \SMS (different from the \NSGA in many-objective optimization) cannot forget Pareto-optimal solution values when the population size is sufficiently large (Lemmas~\ref{lem:sur} and~\ref{lem:surH}). We are convinced that these will find applications in future runtime analyses of the \SMS. In fact, Lemma~\ref{lem:sur} has already been used in the subsequent work~\cite{WiethegerD24}.

This paper is organized as follows. Section~\ref{sec:pre} briefly introduces some basic concepts for the multi-objective optimization and the mathematical notations that will be used in this paper. The many-objective \ojzj benchmark is proposed in Section~\ref{sec:mojzj}. Sections~\ref{sec:sms} to~\ref{sec:htm} analyze the runtime of the original \SMS, the \SMS with stochastic population update, and the one with heavy-tailed mutation on this many-objective benchmark. The runtime for two well-known bi-objective functions is conducted in Section~\ref{sec:LOTZ}. Section~\ref{sec:con} concludes this work.

\section{Preliminaries}\label{sec:pre}

In multi-objective optimization, one tries to find good solutions for a problem containing several objectives. This work considers \emph{pseudo-Boolean maximization}, hence our problem is described by a function $f=(f_1,\dots,f_m):\{0,1\}^n \rightarrow \R^m$. Here $m$ is the number of objectives and $n \in \N$ is called the \emph{problem size}. We use the term \emph{bi-objective optimization} when $m=2$, and \emph{many-objective optimization} when $m\ge 3$. The literature is not uniform in from which number of objectives to speak of many objectives. We include already $m=3$ into the many-objective setting since the drastically different behavior of the NSGA-II detected in~\cite{ZhengD24many} starts already at $m=3$.

Different from single-objective optimization, usually not all solutions are comparable in multi-objective problems. We say that a solution $x \in \{0,1\}^n$ \emph{dominates} a solution $y$ (with regard to~$f$), denoted as $x \succ y$, if $f_i(x)\ge f_i(y)$ for all $i \in \{1,\dots,m\}$, and at least one of these inequalities is strict. For a given multi-objective function $f$, we say that $x\in \{0,1\}^n$ is a \emph{Pareto optimum} if and only if there is no $y\in \{0,1\}^n$ that dominates $x$ (with regard to~$f$). The set of all Pareto optima is called \emph{Pareto set}, and the set of all function values of Pareto optima is called \emph{Pareto front}. The goal for an MOEA is to find the full Pareto front as good as possible, that is, to compute a not too large set $S$ of solutions such that $f(S)$ equals or approximates well the Pareto front. As common in the mathematical runtime analysis of MOEAs, we call the \emph{runtime} of an algorithm the number of function evaluations until its population $P$ covers the full Pareto front, that is, $f(P)$ contains the Pareto front. We refer to~\cite{NeumannW10,AugerD11,Jansen13,ZhouYQ19,DoerrN20} for general introductions to the mathematical runtime analysis of evolutionary algorithms, and to~\cite{Brockhoff11bookchapter} for a discussion of runtime analyses of MOEAs.

\textbf{Mathematical notation.} 
For $a,b\in\N$ with $a\le b$, we denote the set of integers in the interval $[a,b]$ by $[a..b] = \{a, a+1,\dots,b\}$. For a point $x\in\{0,1\}^n$, let $x_{[a..b]}:=(x_{a},x_{a+1},\dots,x_b)$ and $|x|_1$ be the number of ones in $x$.


\section{The $m$-Objective \jump Problem}
\label{sec:mojzj}
As mentioned earlier, the \NSGA was proven to have enormous difficulties in optimizing  many-objective problems~\cite{ZhengD24many}. In that paper, the $m$-objective counterpart \momm of the bi-objective \omm benchmark~\cite{GielL10} was proposed and analyzed. 
Since the first runtime work of the \SMS so far~\cite{BianZLQ23} analyzed the bi-objective jump problem \ojzj, whose special case with gap size $k=1$ is (essentially) the \omm problem, we shall propose and work with an $m$-objective version of \ojzj. Again, its special case $k=1$ will be (essentially) equal to the \momm problem. With this, our results are comparable both the ones in \cite{ZhengD24many} and \cite{BianZLQ23}.

\subsection{\mojzj}
We first recall the definition of the \momm problem. For the ease of presentation, we only consider even numbers $m$ of objectives here. In the \momm problem, the bit string (of length~$n$) is divided into $m/2$ blocks of equal length $2n/m$. On each of these, a bi-objective \omm problem is defined. We note that this general approach to construct many-objective problems from bi-objective ones goes back to the seminal paper of \cite{LaumannsTZ04}.

\begin{definition}[\cite{ZhengD24many}]\label{def:momm}
Let $m$ be the even number of objectives. Let the problem size $n$ be a multiple of $m/2$. Let $n'=2n/m \in \N$. For any $x=(x_1,\dots,x_n)$, the $m$-objective function $\momm$ is the function $f:\{0,1\}^n \rightarrow \R^m$ defined by
\begin{align*}
f_i(x)&=
\begin{cases}
    \om(\bar{x}_{[\frac{i-1}{2}n'+1..\frac{i+1}{2}n']}), & \text{if $i$ is odd,}\\
    \om({x}_{[\frac{i-2}{2}n'+1..\frac{i}{2}n']}), & \text{else,}
\end{cases}
\end{align*}
where $\bar{x}=(1-x_1,\dots,1-x_n)$ and the function $\om: \{0,1\}^{n'} \rightarrow \R$ is defined by
\begin{align*}
\om(y)=\sum_{i=1}^{n'} y_i
\end{align*}
for any $y\in\{0,1\}^{n'}$.
\end{definition}

We define the $m$-objective \ojzj in a similar manner, that is, we divide the $n$ bit positions into $m/2$ blocks and define a \ojzj problem in each block.
\begin{definition}\label{def:mojzj}
Let $m$ be the even number of objectives. Let the problem size $n$ be a multiple of $m/2$. Let $n'=2n/m \in \N$ and $k\in [1..n']$. For any $x=(x_1,\dots,x_n)$, let the $m$-objective function $\mojzj_{k}$ be the function $f:\{0,1\}^n \rightarrow \R^m$ defined by
\begin{align*}
f_i(x)&=
\begin{cases}
\jump_{n',k}(x_{[\frac{i-1}{2}n'+1..\frac{i+1}{2}n']}), & \text{if $i$ is odd,}\\
\jump_{n',k}(\bar{x}_{[\frac{i-2}{2}n'+1..\frac{i}{2}n']}), & \text{else,}
\end{cases}
\end{align*}
where $\bar{x}=(1-x_1,\dots,1-x_n)$ and the function $\jump_{n',k}: \{0,1\}^{n'} \rightarrow \R$ is defined by
\begin{align*}
\jump_{n',k}(y)&=
\begin{cases}
k+\onemax(y), &\text{if $|y|_1\le n'-k$ or $y=1^{n'}$}\\
n'-\onemax(y), & \text{else}
\end{cases}
\end{align*}
for any $y\in\{0,1\}^{n'}$.
\end{definition}
We note that the function $\jump_{n',k}$ used in the definition above is the famous \jump benchmark. It was first defined by Droste, Jansen, and Wegener~\cite{DrosteJW02} and has quickly become the most employed multimodal benchmark in the theory of randomized search heuristics, leading to many fundamental results on how these algorithms cope with local optima~\cite{JansenW02,DangFKKLOSS18,Doerr21cgajump,BenbakiBD21,FajardoS22,RajabiW22,Witt23,DoerrEJK24}.

Obviously, \mojzj  in the case $k=1$ and \momm are equivalent problems, and \mojzj is the previously defined \ojzj problem when $m=2$.

\subsection{Characteristics}
We now give more details on this \mojzj function. Let $B_i:=[(i-1)n'+1..in']$, $i\in[1..m/2]$, be the $i$-th block of the $n$ bit positions. From Definition~\ref{def:mojzj}, we know that the bit values in block $B_i$ only influence the objectives $f_{2i-1}$ and~$f_{2i}$. Figure~\ref{fig:blockplot} plots the objective values of $f_{2i-1}$ and $f_{2i}$ relative to the number of ones in this block. Obviously, \mojzj is multimodal with respect to the definition of multimodality of multi-objective problems in~\cite{ZhengD23ecj}. 

\begin{figure}[!hb]
\centering
\includegraphics[width=0.65\columnwidth]{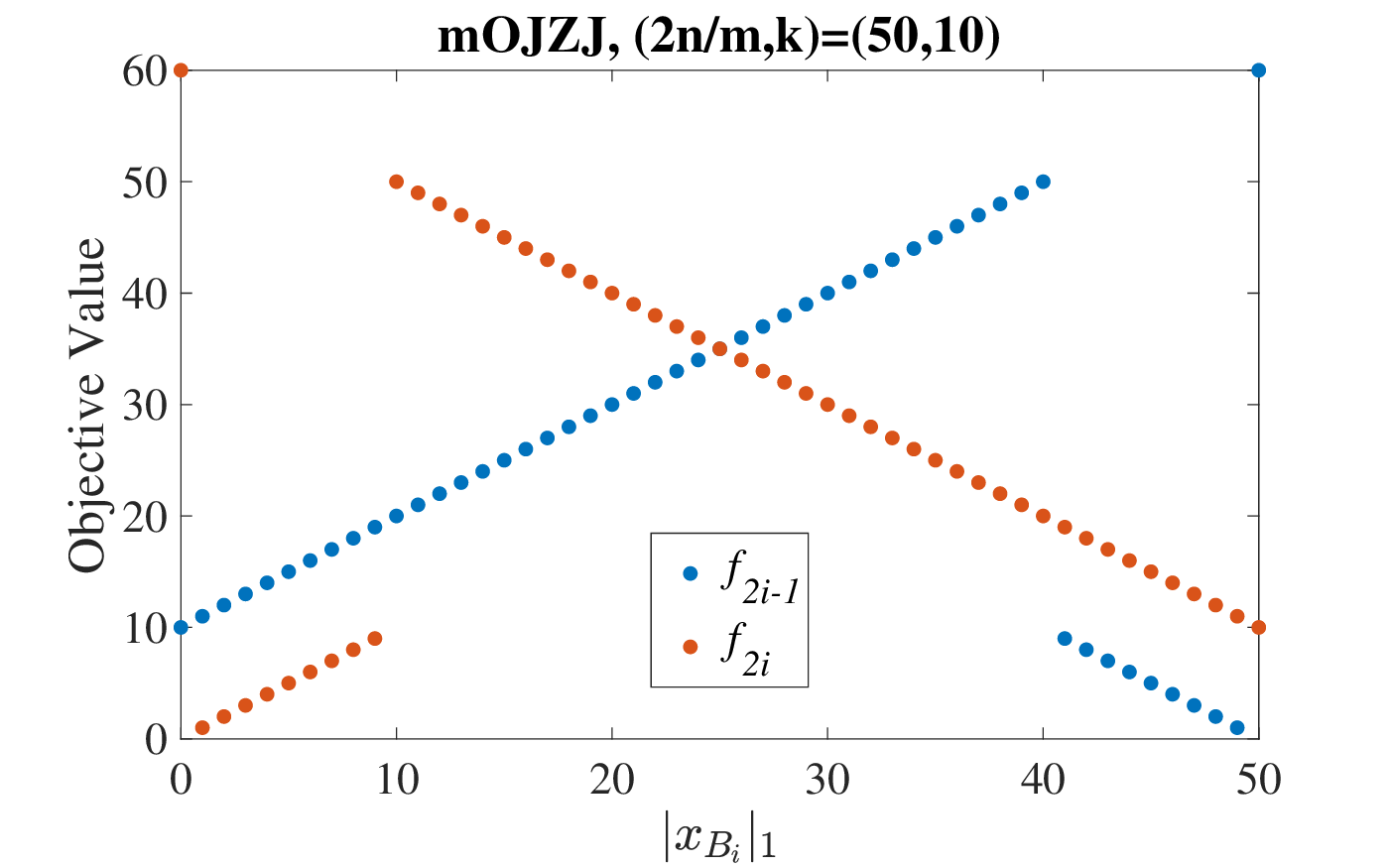} 
\caption{The objective values of $f_{2i-1}$ and $f_{2i}$ in \mojzj w.r.t. $|x_{B_i}|_1$, the number of ones in the block $B_i$.}
\label{fig:blockplot}
\end{figure}

It is not difficult to see that the Pareto set is 
\begin{align*}
S^*=\{x\in\{0,1\}^n \mid \forall i\in[1..m/2], |x_{B_i}|_1 \in [k..n-k] \cup \{0,n\}\}.
\end{align*} 
Hence, the Pareto front is 
\begin{align*}
F^*:=\{&{}(a_1,n'+2k-a_1,\dots,a_{m/2},n'+2k-a_{m/2}) \\
&{}\mid a_1,\dots,a_{m/2}\in[2k..n']\cup\{k,n'+k\}\}
\end{align*}
and the Pareto front size is $M:=|F^*|=(n'-2k+3)^{m/2}$. 

What is harder to determine is the size of a largest set of \emph{incomparable}, that is, pairwise not weakly dominating, individuals. One could be tempted to believe that this number is equal to the size of the Pareto front, but this is not true.\footnote{In fact, in the conference version~\cite{ZhengD24} of this work we made this wrong claim. We sincerely apologize to the research community for this mistake.} Since this number is important in the analysis of MOEAs, for example, it is the best upper bound on the population size of the GSEMO algorithm, we  now discuss this question in more detail.

To this aim, let us define $\overline M$ as the \emph{maximum size of a set of incomparable individuals} (with respect to the \mojzj problem for some parameters $n$, $k$, and $m$). By definition, $\overline M \ge M$, since the Pareto front is witnessed by a set of incomparable solutions. 

The following example shows that $\overline M$ can be considerably larger than $M$. Let $n$ be large, $m = 4$, and $k = n'/2$. Then $F^* = \{(a_1, n'+2k-a_1, a_2, n'+2k-a_2) \mid a_1,a_2 \in \{n'/2, n', 3n'/2\}\}$ and hence $|F^*|=9$. We now construct a set of pairwise incomparable solutions with much larger cardinality. For $i = 1, \dots, k-1$ consider the individual $x^{(i)} = 1^i 0^{n'-i} 1^{n'/2+i} 0^{n'/2-i}$, where we use the word notation known from theoretical computer science to describe bit-strings. In this notation, for example, $1^i$ is a string consisting of $i$ ones, and the concatenation of two strings $u$ and $v$ is denoted by $uv$. We have $f(x^{(i)}) = (k+i,i,n'/2-i,k+n'/2-i)$. Consequently, for $i,j \in [1..k-1]$ with $i \neq j$, the two individuals $x^{(i)}$ and $x^{(j)}$ are incomparable. This shows that $\overline M \ge k-1 = n/4 -1$, which can be considerably larger than $M = 9$. 

Clearly, this example is on the extreme side. Also, noting that all our runtimes will be at least of order $n^k$, such large values of $k$ lead to instance that are too difficult to be solved anyway. For the more typical case that $k$ is small, the following trivial upper bound is relatively tight, as wittnessed by its closeness to the lower bound $M$.

\begin{lemma}
Let $k \le n'/2$ and $P$ be a set of solutions pairwise incomparable solutions 
(w.r.t. \mojzj). Then $|P| \le \overline M \le (n'+1)^{m/2}$.
\label{lem:nonsize}
\end{lemma}

\begin{proof}
  Let $P'$ be a set of pairwise incomparable solutions of maximal cardinality.  By definition, $P'$ contains no two solutions with identical objective value. Consequently, $|P| \le \overline M = |P'| \le f(\{0,1\}^n) = (n'+1)^{m/2}$.  
\end{proof}

For the ease of reading of the proofs in the following sections, we call a Pareto optimum $x$ an \emph{inner Pareto optimum} if for all blocks $B_i,i\in[1..m/2]$, we have $x_{B_i}\notin\{0^{n'},1^{n'}\}$. This is equivalent to saying that $|x_{B_i}|_1\in[k..n'-k]$ for all $i\in [1..m/2]$. We call two objective vectors $u=(u_1,n'+2k-u_1,\dots,u_{m/2},n'+2k-u_{m/2})$ and $v=(v_1,n'+2k-v_1,\dots,v_{m/2},n'+2k-v_{m/2})$ \emph{neighbors} if there exists an $i\in [1..m/2]$ such that $|u_i-v_i|=1$ and $u_j= v_j$ for all $j\ne i$.


\section{The \SMS Can Optimize \mojzj}
\label{sec:sms}
As discussed before, the original \NSGA with non-dominated sorting and crowding distance needs at least exponential time to cover the full Pareto front of \momm, $m \ge 3$, \cite{ZhengD24many}, which is a special case of \mojzj. In this section, we analyze the runtime of the \SMS on this problem, and more generally, the \mojzj problem for $k\in[1..n'/2]$, and show that it does not encounter such problems.

\subsection{Algorithm Description}
\label{ssec:smsalg}
The \SMS is a steady-state variant (that is, the offspring population size is much smaller than the parent population size) of the \NSGA. Like the \NSGA, the \SMS works with a (parent) population of fixed size~$\mu$. However, in each iteration, only one offspring $x'$ is generated from the  population $P_t$. From the $\mu+1$ individuals in the combined parent and offspring population $R_t=P_t\cup \{x'\}$, a single individual is removed. To this aim, the \SMS like the \NSGA uses non-dominated sorting, that is, it partitions $R_t$ into fronts $F_1,\dots,F_{i^*}$, where $F_i$ contains all non-dominated individuals in $R_t\setminus (\bigcup_{j=1}^{i-1} F_j)$. Different from the original \NSGA, which uses the crowding distance as the secondary criterion for the removal of individuals, the \SMS removes the individual with the smallest hypervolume contribution in the critical front (which here always is the last front $F_{i^*}$). 

The \emph{hypervolume} of a set $S$ of individuals w.r.t.\ a reference point $r$ in the objective space is defined as
\begin{align*}
\HV_r(S)= \mathcal{L}\left(\,\bigcup_{u\in S} \{h\in\R^m \mid r \le h \le f(u)\}\right),
\end{align*}
where $\mathcal{L}$ is the Lebesgue measure. The \emph{hypervolume contribution} of an individual $x\in F_{i^*}$ is calculated via
\begin{align*}
\Delta_r(x,F_{i^*}):=\HV_r(F_{i^*})-\HV_r(F_{i^*}\setminus \{x\}).
\end{align*}
Algorithm~\ref{alg:sms-emoa} gives the pseudocode of the \SMS regarded in this work. We generate the offspring by mutating a single, randomly chosen parent via bit-wise mutation with mutation rate $\frac1n$. We note that all our proofs are robust to adding crossover with constant rate $p<1$, that is, if we could generate the offspring via some crossover operator with probability $p$ and via mutation otherwise, all our runtime guarantees would remain valid except for an additional $\frac{1}{1-p}$ factor. 

\begin{algorithm}[t]
\caption{\SMS}\label{alg:sms-emoa}
\begin{algorithmic}[1]
\STATE Initialize $P_0$ by generating $\mu$ solutions uniformly at random from $\{0,1\}^n$
\FOR {$t=0,1,2,\dots,$} 
\STATE Select a solution $x$ uniformly at random from $P_t$
\STATE Generate $x'$ by flipping each bit of $x$ independently with probability $1/n$
\STATE Use fast-non-dominated-sort()~\cite{DebPAM02} to partition $R_t=P_t\cup \{x'\}$ into $F_1, \dots, F_{i^*}$ \label{stp:fronts}
\STATE Calculate $\Delta_r(z,F_{i^*})$ for all $z\in F_{i^*}$ and find $D=\arg\min_{z\in F_{i^*}}{\Delta_r(z,F_{i^*})}$
\STATE Uniformly at random pick $z'\in D$ and set $P_{t+1}=R_t\setminus \{z'\}$ \label{stp:rm}
\ENDFOR
\end{algorithmic}
\end{algorithm}

\subsection{Runtime of \SMS}
\label{ssec:sms}
We now analyze the runtime of the \SMS, that is, the time until its population covers the full Pareto front of \mojzj. 
We start by proving the following general result ensuring that if the population size is at least the maximum size of a set of incomparable solutions, then individuals survive in the population or are replaced by at least as good individuals. We note that our assumption on the population size is standard in mathematical runtime analyses of MOEAs, this or stronger assumptions were made also in the previous analysis of the SMS-EMOA \cite{BianZLQ23} as well as in essentially all analyses of the \NSGA, \NSGAthree, and SPEA2. We formulate this result in manner more general than what we need since we expect this version to be useful in future analyses of the \SMS (and this happened already in~\cite{WiethegerD24}).

The key to the proof is first noting that by a simple domination argument we can concentrate on individuals in $F_1$. 
For these, we argue that (i)~if the combined parent and offspring population contains an individual not in $F_1$, then all solutions in~$F_1$ will survive, and (ii)~if $F_1$ contains two or more solutions with same function value, then only such a solution can be removed from $F_1$. If the population size is large enough, by the pigeon-hole principle, one of these two cases is always satisfied.

\begin{lemma}\label{lem:sur}
Consider any $m$-objective optimization problem. Let $\overline{M} \in \N$ be such that any set $S$ of incomparable solutions  satisfies $|S| \le \overline{M}$. Consider solving this problem via the \SMS with population size $\mu \ge \overline{M}$ and using a reference point $r$ such that $\HV_r(\{x\})>0$ for any individual $x \in \{0,1\}^n$.

Then the following is true. If at some time $t$ the combined parent and offspring population $R_t$ of the \SMS contains some solution $x$ (and thus in particular if $x \in P_t$), then at any later time $s > t$ the parent population $P_s$ contains a solution $y$ such that $y \succeq x$. 

In particular, if $R_t$ contains a Pareto optimum $x$, then all future generations contain a solution $y$ with $f(y) = f(x)$. 
\end{lemma}

\begin{proof}
  Let $t \in \N$ and $x \in R_t$. By induction, it suffices to show the lemma for $s = t+1$, that is, we now show that $P_{t+1}$ contains an individual $y$ such that $y \succeq x$. Without loss of generality, we may assume that $x$ is contained in the first front $F_1$ of the non-dominated sorting of $R_t$ -- if not, then there is some $\tilde x \in F_1$ with $\tilde x \succ x$ and showing the existence of a $y \in P_{t+1}$ with $y \succeq \tilde x$ immediately gives the desired $y$ with $y \succeq x$. 
	
	Hence let us assume that $x \in F_1$. If $|F_1| \le \mu$, then all individuals of $F_1$ are taken into the next generation. Consequently $y := x \in P_{t+1}$ is the desired individual. If $|F_1| = \mu+1$, then by our assumption $\mu \ge \overline{M}$ and the fact that $F_1$ is a set of incomparable solutions, we see that $F_1$ contains at least two individuals with identical objective value (pigeon-hole principle). Since such individuals have a hypervolume contribution of zero, and since only such individuals have a hypervolume contribution of zero, we know that the single individual $z'$ removed from $R_t$ is an individual such that there is at least one more individual $z'' \in R_t$ with $f(z') = f(z'')$. If $z' = x$, then $y := z'' \in P_{t+1}$ is the desired individual. If $z' \neq x$, then $x$ survives into $P_{t+1}$ and thus trivially $y := x$ is the desired individual.
\end{proof}

%
%

Recall that an inner Pareto optimum is a Pareto optimum with $x_{B_i}\notin\{0^{n'},1^{n'}\}$ for all blocks $B_i,i\in[1..m/2]$, as defined before. While it is very likely that at least one initial solution is an inner Pareto optimum, this does not happen with probability one, and hence in Lemma~\ref{lem:1stinner} we estimate the time until the population contains an inner Pareto optimum. Since this time usually is much smaller than the time to generate all remaining Pareto optima, we do not care that this estimate could easily be improved.

The key to the proof is to note that changing a block $x_{B_i}$ to a bit-string with between $k$ and $n'-k$ ones is relatively easy. By Lemma~\ref{lem:sur}, an individual with this $f_i$ value will remain in the population. Hence a total of at most $m/2$ such block changes (applied to the right individual) suffice to obtain an inner Pareto optimum. In the proof, we need the following elementary observation, which extends an analogous result from~\cite{BianZLQ23} for the special case that $n'= n$.

\begin{lemma}
Let $k, n', n \in \N$ such that $k \le n'/2$ and $n' \le n$. Then $\binom{n'-i}{k-i}/n^{k-i}$ increases in $i\in[0..k-1]$.
\label{lem:inc}
\end{lemma}
\begin{proof}
Recall that in the proof of~\cite[Theorem~1]{BianZLQ23}, it was already proven that $\binom{n-i}{k-i}/n^{k-i}$ increases in~$i$. Hence 
\begin{align*}
\frac{\binom{n'-i}{k-i}}{n^{k-i}}=\frac{\binom{n'-i}{k-i}}{(n')^{k-i}}\left(\frac{n'}{n}\right)^{k-i}
\end{align*}
is the product of two expressions both increasing in $i$, and hence is increasing as well.
\end{proof}

Now we state and prove our result on the time to see first inner Pareto optimum.
\begin{lemma}
Let $k\le n'/2$. Consider using the \SMS with $\mu\ge \overline{M}$ to optimize the \mojzj problem. Then after at most $e\mu (mk/2)^k (1+\ln m)$ iterations in expectation, the population (and also the populations afterwards) contains at least one inner Pareto optimum. 
\label{lem:1stinner}
\end{lemma}

\begin{proof}
Denote by $f = (f_1, \dots, f_m)$ the \mojzj function. This lemma is already proven if the initial population has at least one inner Pareto optimum. Now we consider that the initial population $P$ contains no inner Pareto optimum. 
For any $x\in\{0,1\}^n,$ let $I(x)=\{i \in [1..m/2]\mid |x_{B_i}|_1 \in [k..n'-k]\}$ and $L(x)=|I(x)|$. Let $L(P)=\max\{L(x) \mid x\in P\}$. Then $L(P)\ge 0$ and we consider the first time to reach a population $P$ with $L(P)=m/2$, that is, to reach an inner Pareto optimum. 

We first show that $L(P)$ cannot decrease. Let $S_P=\{x\in P \mid L(x)=L(P)\}$. Let $x\in S_P$. From Lemma~\ref{lem:sur}, we know that for the next population, there is $y$ such that $y \succeq x$. Hence, $f_{2i-1}(y) \ge f_{2i-1}(x)$ for any $i\in[1..m/2]$. If $|x_{B_i}|_1 \ge k$,
then $f_{2i-1}(x)=k+|x_{B_i}|_1$. Thus $f_{2i-1}(y)\ge k+|x_{B_i}|_1 \ge 2k$ and hence $|y_{B_i}|_1 \ge k$ from the definition of \mojzj. An analogous argument based on $f_{2i}$ shows that $|x_{B_i}|_1 \le n'-k$ implies $|y_{B_i}|_1\le n'-k$. Hence $L(y) \ge L(x)$, which shows that $L(P)$ does not decrease.

Now we consider the increase of $L(P)$ when $L(P) < m/2$. Let $x \in P$ with $L(x)=L(P)$. The probability to select $x$ as a parent is  $1/\mu$. For a specific $i\in [1..m/2]\setminus I(x)$, we know that $|x_{B_{i}}|_1\in [0..k-1] \cup [n-k+1..n]$. W.l.o.g., let $|x_{B_{i}}|_1 \in [0..k-1]$. Then the probability of generating an individual $y$ with $|y_{B_{i}}| \in [k..n'-k]$ from $x$ is at least
\begin{align}
\binom{n'-|x_{B_{i}}|_1}{k-|x_{B_{i}}|_1}\left(\frac1n\right)^{k-|x_{B_{i}}|_1}\left(1-\frac{1}{n}\right)^{n-(k-|x_{B_{i}}|_1)}
\ge \binom{n'}{k}\frac{1}{n^{k}}\frac1e \ge \left(\frac{n'}{k}\right)^k\frac{1}{en^k}=\frac{1}{e}\left(\frac{2}{mk}\right)^k,
\label{eq:lbpf}
\end{align}
where the first inequality is from 
Lemma~\ref{lem:inc}, 
and the last equality uses $n'=2n/m$.
Note that $|[1..m/2]\setminus I(x)|=m/2-L(P)$. 
Hence, the probability of increasing $L(P)$ in one iteration is at least
$$
\frac{m/2-L(P)}{\mu} \frac1e\left(\frac{2}{mk}\right)^k,
$$
and thus the expected number of iterations to witness such an improvement is at most $e\mu(mk/2)^k/(m/2-L(P))$. Hence, the expected number of iterations to generate an inner Pareto optimum is at most
$$
\sum_{i=0}^{m/2-1}\frac{e\mu}{m/2-i}\left(\frac{mk}{2}\right)^k\le e\mu \left(\frac{mk}{2}\right)^k (1+\ln m),
$$
where the last inequality uses $\sum_{j-1}^{m/2} \frac1j \le 1+\ln (m/2) \le 1+\ln m$.
\end{proof}

In the next lemma, we consider the stage of covering all inner Pareto front points once at least one such point is in the population. The key to the proof is that as long as we have not yet discovered the full inner Pareto front, there always exists a missing Pareto front point that is a neighbor of a point that is already covered by the population. Hence choosing the right parent and flipping the right single bit suffices to cover the desired point on the Pareto front.
\begin{lemma}
Let $k\le n'/2$. Consider using the \SMS with $\mu\ge \overline{M}$ to optimize the \mojzj problem. Assume that the current population contains at least one inner Pareto optimum. Then after at most $en\mu (n'-2k+1)^{m/2}$ iterations in expectation, all inner Pareto front points are covered.
\label{lem:allinner}
\end{lemma}

\begin{proof}
Let $u\in F^*$ be an uncovered inner Pareto front point that has some neighbor covered by the population $P$. Since such a neighbor is also a Pareto front point, we know that it is maintained in all future generations. The probability to generate an offspring with function value $u$ is at least $\left(1-\frac1n\right)^{n-1}\frac{1}{n\mu}\ge \frac{1}{en\mu}$. Hence the expected number of iterations to witness such an offspring is at most $en\mu$. From Lemma~\ref{lem:sur}, we know that once such an offspring is generated, $u$~will be covered in all future populations. Noting that there are at most $(n'-2k+1)^{m/2}$ inner Pareto front points, we know that after an expected number of at most $en\mu (n'-2k+1)^{m/2}$ iterations, all inner Pareto front points are covered.
\end{proof}

The last stage is to cover the remaining Pareto front points. The following lemma bounds the runtime of this phase. The key of the proof is that we can divide the Pareto front into several levels and any individual in the 
$i$-th level can be 
generated from a point in the $(i-1)$-th level by flipping the right $k$ bits.
\begin{lemma}
Let $k\le n'/2$. Consider using the \SMS with $\mu\ge \overline{M}$ to optimize the \mojzj problem. Assume that the current population covers all inner Pareto front points. Then after at most $e((n'-2k+3)^{m/2}-(n'-2k+1)^{m/2})\mu n^k$ iterations in expectation, the full Pareto front is covered.
\label{lem:full}
\end{lemma}
\begin{proof}
We say that a block $x_{B_i}$ of an individual $x$ is \emph{extreme} if $x_{B_i} \in \{(0,\dots,0),(1,\dots,1)\}$. We divide the Pareto front points into $m/2+1$ levels $L_0,L_1,\dots,L_{m/2}$, where $L_j$ consists of the objective values of all Pareto optima having exactly $j$ extreme blocks. Obviously, all inner Pareto front points belong to the $0$-th level. Note that any point in $j$-th level ($j\ge 1$) can be generated from an individual having a certain objective value in $L_{j-1}$ via flipping specific $k$ bits, leaving the other bits unchanged. Hence, the probability to generate any point in $j$-th level condition on that all $(j-1)$-th level points are covered is at least
\begin{align*}
\frac{1}{\mu}\left(1-\frac1n\right)^{n-k}\left(\frac 1n\right)^k \ge \frac{1}{e\mu n^k}.
\end{align*}
From Lemma~\ref{lem:sur}, we know that such a Pareto front point is reached in at most $e\mu n^k$ iterations in expectation, and will survive in all future populations. Consider the process of covering all non-inner Pareto front points in the order of increasing level. Since the $0$-th level points are already covered from our assumption, and since there are at most $((n'-2k+3)^{m/2}-(n'-2k+1)^{m/2})$ non-inner Pareto front points, we know that the remaining levels are covered in at most $e((n'-2k+3)^{m/2}-(n'-2k+1)^{m/2})\mu n^k$ iterations in expectation.
\end{proof}

Summing up the runtimes of the three stages from Lemmas~\ref{lem:1stinner} to~\ref{lem:full}, we obtain the following theorem on the runtime of the full coverage of the Pareto front.

\begin{theorem}\label{thm:sms}
Consider the \mojzj problem with problem size $n$, number of objectives $m$, and gap size $k$. Write $n' = 2n/m$ for the block length. Assume that $k \le n'/2$. Denote by $M = (n' - 2k + 3)^{m/2}$ the size of the Pareto front and by $\overline M$ the size of the largest incomparable set (for which we know that $\overline M \le (2n/m+1)^{m/2}$).  Consider using the \SMS with $\mu\ge \overline{M}$ to solve this problem. Then after at most $e\mu (mk/2)^k (1+\ln m)+e\mu Mn^k=O(\mu M n^k)$ iterations ($\mu+e\mu (mk/2)^k (1+\ln m)+e\mu M n^k=O(\mu Mn^k)$ function evaluations) in expectation, the full Pareto front is covered.
\end{theorem}

Since $k \le n'/2 = n/m$ and $M \ge 2^{m/2}$, we easily see that both runtime expressions in the theorem are $O(\mu M n^k)$, even when allowing $k$, $m$ and $\mu$ to depend on~$n$.

\subsection{Runtime of the GSEMO on \mojzj}
\label{ssec:gsemo}

Since our arguments above can easily be extended to analyze the runtime of the GSEMO on the \mojzj, we quickly do so for reasons of completeness. The GSEMO is the multi\-objective counterpart of the single-objective $(1+1)$~evolutionary algorithm. The initial population of the GSEMO consists of a single randomly generated solution. In each iteration, a solution is picked uniformly at random from the population to generate an offspring via standard bit-wise mutation. If this offspring is not dominated by any solution in the population, it is added to the population and all solutions weakly dominated by it are removed. It is not difficult to see that all solutions in the population of the GSEMO are mutually non-dominated. Hence, the population size of the GSEMO is at most $\overline{M}$ by Lemma~\ref{lem:nonsize}.

When analyzing the runtime of the GSEMO on \mojzj, the main difference is that \SMS requires a statement like Lemma~\ref{lem:sur} to ensure that previous progress is not lost via unlucky selection decisions. For the GSEMO, this property follows immediately from the selection mechanism, which keeps all non-dominated solutions. Together with the upper bound $M$ of the population size, we obtain the following theorem.
\begin{theorem}
Let $k\le n'/2$. Consider using the GSEMO to optimize the \mojzj problem. Then after an expected number of at most $e\overline{M} (mk/2)^k (1+\ln m)+eM\overline{M} n^k$ iterations (or $1 + e\overline{M} (mk/2)^k (1+\ln m)+eM\overline{M} n^k$ fitness evaluations), the full Pareto front is covered.
\end{theorem}

\section{Reduced Impact of Stochastic Population Update}\label{sec:spu}

Bian et al.~\cite{BianZLQ23} proposed a stochastic population update mechanism for the \SMS and proved, somewhat unexpectedly given the state of the art, that it can lead to significant performance gains. More precisely, it was proven that the classic \SMS with reasonable population size $\mu=\Theta(n)$ solves the bi-objective \ojzj problem in an expected number of $O(n^{k+2}) \cap \Omega(n^{k+1})$ function evaluations, whereas for the \SMS with stochastic population update, $O(n^{k+2} \min\{1, 2^{-k/4} n\})$ suffice. 
Hence, for $k=\omega(\log n)$, a super-polynomial speed-up was shown. 

In this section, we extend the analysis of \cite{BianZLQ23} to $m$ objectives. Unfortunately, we will observe that the larger population sizes necessary here reduce the impact of the stochastic population update. As in \cite{BianZLQ23}, we have no proven tight lower bounds for the \SMS with stochastic population update, but our upper-bound proofs suggest that the reduced effect of the stochastic population update on the runtime guarantee is real, that is, the impact on the true runtime is diminishing with the larger population sizes necessary in the many-objective setting.

As a side result, we simply and moderately improve the analysis of how the stochastic population update helps to traverse the valley of low-quality solutions of the \ojzj benchmarks. Whereas in~\cite{BianZLQ23} a complex drift argument was used to show a speed-up of $\mu / 2^{k/4}$, we give a simple waiting time argument that yields a speed-up of $\mu / (2^k-1)$, both in the bi-objective case of~\cite{BianZLQ23} and in the general case. 

\subsection{Stochastic Population Update}
\label{ssec:spu}

The rough idea of the stochastic selection proposed by Bian et al.~\cite{BianZLQ23}  is that a random half of the individuals survive into the next generation regardless of their quality. The individual to be discarded is chosen from the remaining individuals according to non-dominated sorting and hypervolume contribution as in the classic \SMS. This approach resembles the random mixing of two acceptance operators in the Move Acceptance Hyper-Heuristic studied recently~\cite{LehreO13,LissovoiOW23,DoerrDLS23}.

More precisely, after generating the offspring $x'$, the \SMS with stochastic selection chooses from $P_t \cup \{x'\}$ a set $R'$ of $\lfloor (\mu+1)/2 \rfloor$ solutions randomly with replacement. The individual $z'$ to be removed from $P_t \cup \{x'\}$ is then determined via non-dominated sorting and hypervolume contribution applied to $R'$ only.  
We note that with this mechanism, any solution enters the next generation with probability at least $1/2$. 

To analyze the runtime of the \SMS with stochastic population update, we first derive and formulate separately two insights on the survival of solutions. The following Lemmas~\ref{lem:surs} is an elementary consequence of the algorithm definition, whereas Lemma~\ref{lem:sursu} builds on Lemma~\ref{lem:sur}. 

\begin{lemma}\label{lem:surs}
Consider using the \SMS with stochastic population update to optimize any specific problem. For any iteration $t$ and $x\in R_t$, we have that there is a $y$ with $f(y)=f(x)$ in $P_{t+1}$ with probability at least $1/2$.
\end{lemma}
\begin{proof}
We simply calculate the probability of $x$ being removed after the population update. Since $x$ can be removed only if it is selected in $R'$, we know that it will be removed with probability at most $\frac{\lfloor (\mu+1)/2 \rfloor}{\mu+1} \le 1/2$, and hence $f(x)$ will be maintained in $P_{t+1}$ with probability at least $1/2$.
\end{proof}

\begin{lemma}\label{lem:sursu}
Consider using the \SMS with stochastic population update and with $\mu\ge 2\overline{M}+1$ to optimize any specific problem, where $\overline{M}$ is the maximum size of a set of incomparable solutions for this problem. 
If at some time $t$ the combined parent and offspring population $R_t$ of the \SMS contains some solution $x$, then at any later time $s > t$ the parent population $P_s$ contains a solution $y$ such that $y \succeq x$. 
\end{lemma}
\begin{proof}
If $x\notin R'$, then it survives automatically into the next generation. If $x \in R'$, then it is subject to a selection as in the standard \SMS. Since $R'=\lfloor (\mu+1)/2 \rfloor > \overline{M}$, Lemma~\ref{lem:sur} gives the desired $y$.  
%
\end{proof}

\subsection{Runtime Analysis for the \SMS With Stochastic Population Update}

We now analyze the runtime of the \SMS with stochastic population update on the many-objective \ojzj problem. We start by distilling why the \SMS profits from the stochastic population update and tightening the analysis of~\cite{BianZLQ23}. Our analysis is also significantly simpler than the previous one, which might easy future works in this direction. 

The core of the analysis in~\cite{BianZLQ23} is proving that the \SMS can profit from the stochastic population update by accepting solutions in the valley of low-quality solutions of the bi-objective \ojzj benchmark. When extracted from the long proof in~\cite{BianZLQ23}, this would give the following lemma.

\begin{lemma}[\cite{BianZLQ23}] \label{lem:drift}
  Consider using the \SMS with stochastic population update to optimize the bi-objective \ojzj benchmark. Assume that $\mu \ge 2(M+1)$. Assume that at some time $t$, the population contains an individual with $k$ ones. Then the expected number of additional iterations to generate the all-zero solution is at most $\min\{e \mu n^k, (e\mu n^{k/2}+1) 2e\mu^2 n^{k/2} / 2^{k/4}\} = e \mu n^k \min\{1, (1+o(1)) 2 e \mu  / 2^{k/4}\}$.
\end{lemma}

The proof of this lemma relies on a technical application of the additive drift theorem \cite{HeY01}, in particular, with a non-intuitive potential function. 

Since the valleys of low-quality solutions of the \ojzj benchmark have the same shape regardless of the number of objectives, we could use the elementary lemmas from the previous subsection to extend the lemma above to the many-objective setting. Since we are not fully satisfied with the complex proof of Lemma~\ref{lem:drift}, we propose an alternative one that appears both simpler and gives a tighter result. We formulate it in the many-objective setting, but note that it includes the bi-objective setting. In terms of the bound, our result replaces the factor of $(1+o(1))2^{-k/4}$, which is the term responsible for the runtime improvement, by a factor of $2^{-k+1}$, so essentially the fourth power of the previous factor.

\begin{lemma}
  Consider using the \SMS with stochastic population update to optimize the $m$-objective \ojzj benchmark, $m \ge 2$. Assume that $\mu \ge 2\overline{M}+1$. Assume that at some time $t$, the population contains an individual $x$ with $k$ ones in the $i$-th block. Then the expected number of additional iterations to generate an individual $z$ with only zeros in the $i$-th block and with $f_j(z) = f_j(x)$ for all $j \in [1..m] \setminus \{2i-1,2i\}$  is at most $e \mu n^k \min\{1, 4 e \mu  / 2^k\}$.
\end{lemma}

\begin{proof}
  The proof of this lemma consists of a simple waiting time argument, that is, we compute the probability $p$ that the desired solution $z$ is generated from $x$ in at most two iterations and then bound the expected waiting time by $2/p$. 
	
	Let $x$ be as in the assumptions of this lemma. Let $X$ be the set of indices of ones it has in the $i$-th block. Let $Y \subseteq X$. Let $E_Y$ be the event that in the first iteration, (i)~$x$ is selected as parent for mutation, (ii)~it is mutated into an individual $y$ that agrees with $x$ in all bits except for exactly the ones in $Y$, and (iii)~this $y$ is taken into the next parent population (despite being dominated by $x$ when $Y \notin \{\emptyset,X\}$), and in the next iteration, (iv)~this $y$ is selected as parent for mutation and (v)~mutated into a solution $z$ that agrees with $x$ in all bits except for exactly the ones in~$X$. The probabilities for these steps to happen as described are $1/\mu$ for step (i) and (iv), $(1/n)^{|Y|} (1-1/n)^{n-|Y|}$ for step~(ii), at least $1/2$ for step~(iii) by Lemma~\ref{lem:surs}, and $(1/n)^{k-|Y|} (1-1/n)^{n-(k-|Y|)}$ for step~(v), noting that here exactly the bits in $X \setminus Y$ have to be flipped. Multiplying these probabilities, we see that $\Pr[E_Y] \ge \tfrac 12 (1-1/n)^{2n-k} \mu^{-2} n^{-k} \ge \tfrac 12 e^{-2} \mu^{-2} n^{-k}$. There are $2^k$ disjoint events~$E_Y$, each ensuring that after two iterations the population contains an individual $z$ as desired. Hence the probability $p$ to obtain such a $z$ within two iterations is at least $p \ge \sum_Y \Pr[E_Y] \ge 2^k \cdot \tfrac 12 e^{-2} \mu^{-2} n^{-k}$.
	
	By Lemma~\ref{lem:sursu}, the population always contains an individual satisfying the assumptions on~$x$ and having the same objective value as~$x$. Hence the time $T$ to generate the desired individual~$z$ is stochastically dominated by two times a geometric distribution with success rate~$p$, hence $E[T] \le 2/p \le 4 e^2 \mu^2 n^k / 2^k$.	
\end{proof}

Since the proofs of Lemmas~\ref{lem:1stinner} and~\ref{lem:allinner} for the classic \SMS mostly relied on elementary properties of standard bit-wise mutation and on the survival guarantee of Lemma~\ref{lem:sur}, we can now use the survival guarantee of Lemma~\ref{lem:sursu} to obtain analogous results for the \SMS with stochastic selection (at the price of requiring essentially twice the population size). 
This yields the following estimates for the time to obtain at least one inner Pareto optimum and the time to cover all inner Pareto front points starting from a population with at least one inner Pareto optimum.

\begin{lemma}
\label{lem:innerspu}
Let $k\le n'/2$. Consider using the \SMS with stochastic population update and with $\mu\ge 2\overline{M}+1$ to optimize the \mojzj problem. Then 
\begin{itemize}
\item after at most $e\mu (mk/2)^k (1+\ln m)$ iterations in expectation, the population (and also the populations afterwards) contains at least one inner Pareto optimum;
\item after another at most $en\mu (n'-2k+1)^{m/2}$ iterations in expectation, all inner Pareto front points are covered.
\end{itemize}
\end{lemma}

Now we consider the runtime for the full coverage of the Pareto front after all inner Pareto front points are covered. This analysis follows precisely  the proof of Lemma~\ref{lem:full} for the original \SMS except that we replace the estimate for expected time to obtain an extremal block, $e\mu n^k$ in that lemma, by the potentially lower number $e \mu n^k \min\{1, 4 e \mu  / 2^k\}$. This immediately gives the following result.

\begin{lemma}\label{lem:fullspu}
Let $k\le n'/2$. Consider using the \SMS with stochastic population update and with $\mu\ge 2\overline{M}+1$ to optimize the \mojzj problem. Assume that the current population covers all inner Pareto front points. Then after at most $e\mu ((n'-2k+3)^{m/2}-(n'-2k+1)^{m/2}) n^k \min\{1,4e\mu / 2^k\}$ iterations in expectation, the full Pareto front is covered.
\end{lemma}

Combining Lemmas~\ref{lem:innerspu} and~\ref{lem:fullspu}, we have the runtime of the \SMS with stochastic population update in the following theorem. It includes the case of two objectives, and then, with $\overline M = M = (n-2k+3)$, improves the previous result for this case, Theorem~2 in~\cite{BianZLQ23}, in particular by replacing the $2^{k/4}$ term by~$2^k$. 

\begin{theorem}
Let $k\le n'/2$. Consider using \SMS with $\mu\ge 2\overline{M}+1$ and stochastic population update to optimize the \mojzj problem. Then after at most $e\mu (mk/2)^k (1+\ln m)+{en\mu (n'-2k+1)^{m/2}+e\mu((n'-2k+3)^{m/2}-(n'-2k+1)^{m/2}) n^k \min\{1,4e\mu / 2^k\}} = O(\mu M n^k \min\{1,4e\mu / 2^k\})$ iterations in expectation, the full Pareto front is covered.
\label{thm:spu}
\end{theorem}

Comparing the runtime guarantees of Theorem~\ref{thm:sms} (classic \SMS) and
Theorem~\ref{thm:spu} (\SMS with stochastic selection), we see that stochastic selection can at most lead to a speed-up by a factor of order $2^{k} / \mu$. Now $\mu$ is at least $\Omega(M) = \Omega((2n/m-2k+3)^{m/2})$. Consequently, the advantage of stochastic selection is rapidly decreasing with growing numbers of objectives and vanishes, e.g., when $m \ge k$. 

\section{Heavy-Tailed Mutation Helps}
\label{sec:htm}

In the previous section, we saw that the advantage of the stochastic population update does not generalize well from bi-objective to many-objective optimization. We now regard another design choice that so far was only analyzed in bi-objective optimization, namely a heavy-tailed mutation operator. For this, we shall prove that the $k^{\Omega(k)}$ factor speed-up observed in bi-objective optimization extends to many objectives.

\subsection{Heavy-Tailed Mutation}\label{ssec:htm}
Different from the standard bit-wise mutation operator, which flips each bit independently with probability $1/n$, the heavy-tailed mutation operator proposed in \cite{DoerrLMN17} flips each bit independently with probability $\alpha/n$, where $\alpha$ follows a power-law distribution with parameter $\beta$. The number $\alpha$ is sampled anew in each application of the heavy-tailed mutation operator. The underlying power-law is defined as follows.

\begin{definition}
Let $n\in\N$ and $\beta >1$. We say that $\alpha$ follows a power-law distribution with (negative) exponent $\beta$ if for all $i\in[1..n/2]$, we have $\Pr[\alpha=i]=\left(C_{n/2}^{\beta}\right)^{-1}i^{-\beta}$, where $C_{n/2}^{\beta}=\sum_{j=1}^{n/2}j^{-\beta}$.
\end{definition}

The power-law in the choice of $\alpha$ extends to the offspring distribution in the way that an offspring has Hamming distance $j$ from the parent with probability $\Omega(j^{-\beta})$~\cite{DoerrLMN17}. This facilitates larger jumps in the search space, and thus escaping from a local optimum. We recall that for standard bit-wise mutation, the offspring has Hamming distance $j$ from the parent with probability is $j^{-\Omega(j)}$ only.

Since this observation was formulated in~\cite[Lemma~5.1]{DoerrLMN17} in a slightly impractical manner (with implicit constants independent of~$\beta$, but only applying to distances in $(\beta-1,n/2]\cup\{1\}$), we first give a formulation that is more useful for most purposes. It treats $\beta$ as a constant, which makes a lot of sense in the light of all previous results, which gave the best performances for constant values of $\beta$. 
\begin{lemma}
Let $\beta > 1$. Then there is a $c > 0$ such that the following holds. Let $n \in \N$. Let $x\in\{0,1\}^n$ and $y$ be generated by applying the heavy-tailed mutation (with parameter $\beta$) to $x$. Let $H(x,y)$ be the Hamming distance between $x$ and $y$, let $j \in [1..n/2]$, and let $P_j^{\beta}:=P_j^\beta(n) := \Pr[H(x,y)=j]$. Then 
\[P_j^{\beta} \ge c j^{-{\beta}}.\]
\label{lem:proht}
\end{lemma}
\begin{proof}
If $\beta-1 > \sqrt{n}$, then any $c>0$ such that $c \le \min\{P_j^\beta(n) \mid n \in [1..(\beta-1)^2], j \in [1..n/2]\}$ trivially satisfied our claim. Hence, in the following, let $\beta-1 \le \sqrt n$.

If $j\in(\beta-1,n/2]\cup\{1\}$, then \cite[Lemma~5.1~(1) and (3)]{DoerrLMN17} gives our result. If $j\in(1,\beta-1]$, then $j \le \beta-1 \le \sqrt{n}$ and we conclude from \cite[Lemma~16]{ZhengD23ecj} that $P_j^{\beta} \ge \frac{\beta-1}{2\sqrt{2\pi}e^{8\sqrt2 + 13}\beta}j^{-\beta}$. Hence $c = \frac{\beta-1}{2\sqrt{2\pi}e^{8\sqrt2 + 13}\beta}$ suffices in this case.
\end{proof}

The heavy-tailed mutation operator has resulted in asymptotic performance gains by a factor of $k^{\Omega(k)}$ for the \oea optimizing single-objective (classic and generalized) \jump functions with gap size $k$~\cite{DoerrLMN17,BamburyBD24}. For the bi-objective \ojzj benchmark, again a speed-up of $k^{\Omega(k)}$ was proven when optimized via the GSEMO~\cite{ZhengD23ecj} and the \NSGA~\cite{DoerrQ23tec}. Several other positive theoretical results exist for this heavy-tailed mutation, or more generally, other heavy-tailed parameter choices~\cite{FriedrichQW18,WuQT18,QuinzanGWF21,CorusOY21tec,DangELQ22,AntipovBD22,DoerrR23,DoerrGI24}. 

Prior to this work, no theoretical analysis of the heavy-tailed mutation operator in many-objective optimization existed.

%

\subsection{Runtime}

We now analyze the runtime of the \SMS using the heavy-tailed mutation operator instead of bit-wise mutation, the standard choice for this algorithm.

Since the heavy-tailed mutation does not change the survival selection of the original \SMS, we immediately have the following result on the survival of the individuals.
\begin{corollary}
The assertion of Lemma~\ref{lem:sur} extends to the \SMS with heavy-tailed mutation.
\label{lem:surH}
\end{corollary}

With proof ideas similar to those in Lemma~\ref{lem:1stinner} to~\ref{lem:full}, we prove the following three runtime estimates. The main differences are the different probabilities for generating a certain solution. However, these require somewhat different computations in the following lemma, which therefore appears slightly technical. We note that we did not optimize the result here, but were content with arriving at a bound that is small compared to the time needed to generate the extremal Pareto optima.

\begin{lemma}
Let $k\le n'/2$. Consider using the \SMS with heavy-tailed mutation and with $\mu\ge \overline{M}$ to optimize the \mojzj problem. Then after at most $O\left((em/2)^k k^{\beta}\mu\ln m\right)$ (for $m\ge 4$) or $O\left(\binom{n}{k} k^{\beta}\mu/(n-k+1)\right)$ (for $m=2$) iterations in expectation, the population (and also the populations afterwards) contains at least one inner Pareto optimum. 
\label{lem:1stinnerH}
\end{lemma}

\begin{proof}
We use the same general as in Lemma~\ref{lem:1stinner} except for the following probability estimates. 
From Lemma~\ref{lem:proht}, we know that the probability to generate a $y$ with $|y_{B_i}|_1 \in [k..n'-k]$ from $x$ with $|x_{B_{i}}|_1 \in [0..k-1]$ is at least
\begin{align}
\Omega\left(\frac{\binom{n'-|x_{B_i}|_1}{k-|x_{B_i}|_1}(k-|x_{B_i}|_1)^{-\beta}}{\binom{n}{k-|x_{B_i}|_1}}\right). 
%
\label{eq:lbf}
\end{align}
Let first $m\ge 4$. For all $j \in [0..k-1]$, we compute
\begin{equation}
\begin{split}
&\hspace*{-0.4cm}\left(\frac{\binom{n'-(j+1)}{k-(j+1)}}{\binom{n}{k-(j+1)}(k-(j+1))^{\beta}}\right)\bigg/\left(\frac{\binom{n'-j}{k-j}}{\binom{n}{k-j}(k-j)^{\beta}}\right)\\
={}&{}\left(\frac{(n'-(j+1))!(k-(j+1))!(n-k+j+1)!}{(k-(j+1))!(n'-k)!n!(k-(j+1))^{\beta}}\right)\\
&{}\bigg/\left(\frac{(n'-j)!(k-j)!(n-k+j)!}{(k-j)!(n'-k)!n!(k-j)^{\beta}}\right)\\
={}&{}\frac{(n'-(j+1))!(n-k+j+1)!(k-j)^{\beta}}{(n'-j)!(n-k+j)!(k-(j+1))^{\beta}}
=\frac{n-k+j+1}{n'-j}\left(\frac{k-j}{k-(j+1)}\right)^{\beta}\\
\ge{}&{} \frac{n-k+1}{n'} \ge \frac{mn'/2-n'/2+1}{n'} >1,
\end{split}
\label{eq:estp}
\end{equation}
where the antepenultimate inequality uses $j\ge 0$, and the last inequality uses $m\ge 4$. Hence the expression in~\eqref{eq:lbf} is smallest for $|x_{B_i}|_1 = 0$, and hence at least of the order of 
\begin{align}
\frac{\binom{n'}{k}}{\binom{n}{k}} k^{-\beta} \ge \left(\frac{n'}{en}\right)^k k^{-\beta}=\left(\frac{2}{em}\right)^k k^{-\beta}.
\label{eq:m4}
\end{align}
For $m=2$, we have $n=n'$. Since 
\[
\binom{n-j}{k-j}=\binom{n-j}{n-k} \ge \binom{n-(k-1)}{n-k}=n-k+1
\]
for $j\in[0..k-1]$, and
\begin{align}
\frac{n-k+1}{\binom{n}{k-|x_{B_i}|_1}}(k-|x_{B_i}|_1)^{-\beta} \ge \frac{n-k+1}{\binom{n}{k}}k^{-\beta},
\label{eq:m2}
\end{align}
we can estimate the expression in~\eqref{eq:lbf} by $\Omega\left((n-k+1)k^{-\beta}/\binom{n}{k}\right)$.

Replacing (\ref{eq:lbpf}) by (\ref{eq:m4}) or (\ref{eq:m2}) in the proof of Lemma~\ref{lem:1stinner}, we obtain that the expected number of iterations to reach an inner Pareto optimum is $O\left((em/2)^k k^{\beta}\mu\ln m\right)$ for $m\ge 4$ and is $O\left(\binom{n}{k} k^{\beta}\mu/(n-k+1)\right)$ for $m=2$.
\end{proof}

\begin{lemma}
Let $k\le n'/2$. Consider using the \SMS with heavy-tailed mutation and with $\mu\ge \overline{M}$ to optimize the \mojzj problem. Assume that the current population contains at least one inner Pareto optimum. Then after at most $\frac{e\beta}{\beta-1} n\mu (n'-2k+1)^{m/2}$ iterations in expectation, all inner Pareto front points are covered.
\label{lem:allinnerH}
\end{lemma}
\begin{proof}
We use similar arguments as in the proof of Lemma~\ref{lem:allinner} except for the following probability calculation. The probability to generate an offspring with function value equal to the one of an uncovered inner Pareto front point from a neighbor covered by the population is at least $\frac{1}{n\mu}P_1^{\beta}\ge \frac{\beta-1}{e\beta n\mu}$, where $P_1^{\beta} \ge \frac{\beta-1}{e\beta}$ is from~\cite[Lemma~16]{ZhengD23ecj}.
\end{proof}

\begin{lemma}
Let $k\le n'/2$. Consider using the \SMS with heavy-tailed mutation and with $\mu\ge \overline{M}$ to optimize the \mojzj problem. Assume that the current population covers all inner Pareto front points. Then after at most $\binom{n}{k}C_{n/2}^{\beta}\mu  k^{\beta}((n'-2k+3)^{m/2}-(n'-2k+1)^{m/2})$ iterations in expectation, the full Pareto front is covered.
\label{lem:fullH}
\end{lemma}

\begin{proof}
We reuse the arguments of Lemma~\ref{lem:full} except for the following probability calculation. From Lemma~\ref{lem:proht}, we know that the probability to generate any point in $j$-th level condition on that all $(j-1)$-th level points are covered is at least 
\begin{equation*}
\frac{1}{\mu}\binom{n}{k}^{-1}P_k^{\beta} = \Omega\left(\binom{n}{k}^{-1}\frac{1}{\mu}k^{-\beta}\right).
\qedhere
\end{equation*}
\end{proof}

Considering the whole process for the full coverage of the Pareto front from Lemmas~\ref{lem:1stinnerH} to~\ref{lem:fullH}, we have the following runtime result for the \SMS with heavy-tailed mutation (where we estimated $\binom{n}{k} \le n^k/(k!)$ and $k!\ge \sqrt{2\pi} k^{k+0.5}e^{-k}$). 
\begin{theorem}\label{thm:htm}
Let $k\le n'/2$ and $\beta >1$. Consider using \SMS with $\mu\ge \overline{M}$ and with heavy-tailed mutation to optimize the \mojzj problem. Then after at most $O(M\mu k^{\beta-0.5-k}(en)^k)$ iterations or function evaluations in expectation, the full Pareto front is covered.
\end{theorem}

Compared to the runtime guarantee of Theorem~\ref{thm:sms} for the original \SMS, the guarantee of Theorem~\ref{thm:htm} above for the \SMS with heavy-tailed mutation is by a factor of asymptotically $k^{k+0.5-\beta}/e^k$ stronger. Compared to the \SMS with stochastic population update (Theorem~\ref{thm:spu}), the speed-up is by a factor of $\min\{k^{k+0.5-\beta}/e^k,4\mu k^{0.5-\beta}(k/2)^k/e^{k-1}\}$, which is super-exponential w.r.t. $k$.

\section{Runtime Analysis for Two Classic Bi-objective Benchmarks}
\label{sec:LOTZ}

In the sections above, we discussed the runtime of the \SMS for the many-objective \ojzj benchmark and showed a good performance, different from the original \NSGA. Since the only two previous theory papers~\cite{BianZLQ23,ZhengLDD24} on the \SMS merely considers its performance on the bi-objective \ojzj and \DLTB problems, to broaden our understanding of this algorithm, we now analyze its runtime for the two most prominent bi-objective benchmarks, \omm and \lotz.

\subsection{\omm and \lotz}
The \omm benchmark introduced by Giel and Lehre~\cite{GielL10} (and likewise the similar \cocz benchmark previously proposed by Laumanns, Thiele, and Zitzler~\cite{LaumannsTZ04}) are bi-objective counterparts of the famous single-objective \om benchmark. Analogously, the \lotz benchmark defined in~\cite{LaumannsTZ04} (also the weighted version WLPTNO of Qian, Yu, and Zhou~\cite{QianYZ13}) are bi-objective counterparts of the classic single-objective \lo benchmark. 

For \omm, the two objectives count the number of ones and zeros in a given bit-string, respectively. For \lotz, the first objective is the number of contiguous ones starting from the first bit position, and the second objective is the number of contiguous zeros starting from the last bit position. Both benchmarks with their simple and clear structure have greatly facilitated the understanding of how MOEAs work. We note that there are other benchmarks used in the mathematical analysis of MOEAs, see, e.g.,~\cite{HorobaN08,BrockhoffFHKNZ09,QianTZ16,LiZZZ16,DangOS24}, but clearly \omm and \lotz are the most prominent ones. 

\begin{definition}[\cite{GielL10,LaumannsTZ04}]
Let $n\in \N$ be the problem size. The \omm and \lotz benchmarks are the functions $\{0,1\}^n \rightarrow \R^2$ defined by
\begin{align*}
\omm(x)&=\left(\sum_{i=1}^n (1-x_i), \sum_{i=1}^n x_i\right) \\
\lotz(x)&=\left(\sum_{i=1}^n \prod_{j=1}^ix_j, \sum_{i=1}^n \prod_{j=i}^{n} (1-x_j)\right)
\end{align*}
for all $x=(x_1,\dots,x_n) \in \{0,1\}^n$.
\end{definition}


\subsection{Runtime}

We now prove runtime guarantees for the \SMS on the \omm and \lotz benchmarks. 
Note that the Pareto front, and, more generally, the size of any set of incomparable solutions for \omm and \lotz have sizes at most $n+1$. Hence we can use Lemma~\ref{lem:sur} with $\overline{M} = n+1$. 


Theorem~\ref{thm:omm} below gives an upper bound for the runtime of \SMS on \omm. This bound (in terms of fitness evaluations) is of the same asymptotic order as the runtime guarantee $\frac{2e^2}{e-1}\mu n(\ln n+1)$ for the \NSGA \cite[Theorem~2]{ZhengD23aij}. Due to the more complex population dynamics, the leading constant in the result in~\cite{ZhengD23aij} is inferior (58\% larger). More critically, the result in~\cite{ZhengD23aij} was proven for population sizes $\mu\ge 4(n+1)$ only, and it was disproven for $\mu=n+1$~\cite[Theorem~12]{ZhengD23aij}, whereas our result here holds for all $\mu \ge n+1$. 

Knowing from Lemma~\ref{lem:sur} that we cannot lose Pareto points, the proof of our result consists of adding the waiting times for finding a new Pareto point from an already existing neighboring one.

\begin{theorem}
Consider using the \SMS with $\mu\ge n+1$ to optimize the \omm problem with problem size~$n$. Then after at most $2e\mu n(\ln n+1)$ iterations (or $\mu + 2e\mu n(\ln n+1)$ fitness evaluations) in expectation, the full Pareto front is covered.
\label{thm:omm}
\end{theorem}
\begin{proof}
Since each search point in $\{0,1\}^n$ is Pareto optimal, from Lemma~\ref{lem:sur}, we know that any Pareto front point covered by the population $P_t$ will be maintained in all future generations. Assume that for some $i\in[0..n]$, the population $P_t$ contains a solution with objective value $(i,n-i)$. Then the probability to reach a solution with value $(i+1,n-i-1)$ is at least
\begin{align*}
\frac{1}{\mu} \frac{n-i}{n} \left(1-\frac1n\right)^{n-1} \ge \frac{n-i}{e\mu n},
\end{align*}
and the probability to reach a solution with function value $(i-1,n-i+1)$ is at least
\begin{align*}
\frac{1}{\mu} \frac{i}{n} \left(1-\frac1n\right)^{n-1} \ge \frac{i}{e\mu n}.
\end{align*}
Thus the expected numbers of iterations to cover $(i+1,n-i-1)$ and $(i-1,n-i+1)$ are at most $\frac{e\mu n}{n-i}$ and at most $\frac{e\mu n}{i}$, respectively. Hence, the expected number of iterations to cover the full Pareto front is
\begin{equation*}
\Bigg(\sum_{j=i}^{1} \frac{e\mu n}{n-j}\Bigg) +\left(\sum_{j=i}^{n-1} \frac{e\mu n}{j}  \right)
\le 2\sum_{j=1}^{n-1} \frac{e\mu n}{j} \le 2e\mu n(\ln n+1).
\qedhere
\end{equation*}
\end{proof}

We now turn to the \lotz benchmark. Our runtime guarantee, see Theorem~\ref{thm:lotz} below, is again of the same asymptotic order of magnitude as the guarantee $\frac{2e^2}{e-1}\mu n^2$ proven for the \NSGA \cite[Theorem~8]{ZhengD23aij}, and again, the guarantee for the \NSGA has a weaker leading constant and requires a larger population size of $\mu \ge 4(n+1)$. In the proof, 
we first regard the time to reach the individual~$1^n$, and then consider the times to generate the individuals $1^{i}0^{n-i},i=n-1,\dots,0$ one after the other. In both stages, we heavily rely on Lemma~\ref{lem:sur} asserting that important progress is not lost.

\begin{theorem}
Consider using the \SMS with $\mu\ge n+1$ to optimize the \lotz problem with problem size $n$. Then after at most $2e\mu n^2$ iterations (or $\mu + 2e\mu n^2$ fitness evaluations) in expectation, the full Pareto front is covered.
\label{thm:lotz}
\end{theorem}

\begin{proof}
Denote by $f$ the \lotz function. We first consider the time to reach $1^n$. Let $X_t$ denote the maximum $f_1$-value of an individual in $P_t$. Assume that at some $t$, we have $X_t < n$. Let $x \in P_t$ be an individual with $f_1(x) = X_t$. Then with probability at least $\frac 1 \mu (1-\frac 1n)^{n-1} \frac 1n$, this $x$ is chosen as parent and mutated to an individual $y$ with $f_1(y) > X_t$. By Lemma~\ref{lem:sur}, $P_{t+1}$ contains an individual $z$ with $z \succeq y$. In particular, we have $f_1(z) \ge f_1(y) > X_t$ and hence $X_{t+1} > X_t$. 

We have just seen that the probability to increase $X_t$ by at least one is at least $\frac 1 \mu (1-\frac 1n)^{n-1} \frac 1n \ge \frac{1}{e\mu n}$. Consequently, the expected waiting time for this event is at most $e\mu n$. After at most $n$ such increases, hence after an expected time of at most $e\mu n^2$, we have reached the state $X_t = n$, that is, $1^n \in P_t$.

Once $1^n$ is in the population, we consider the time to reach the Pareto optima $1^{n-1}0,\dots,10^{n-1},0^n$ one by one. Assume that $1^i0^{n-i}$ with $i\in[1..n]$ is in the population. Then to reach $1^{i-1}0^{n-i+1}$, we only need to choose $1^i0^{n-i}$ as parent and flip the $i$-th bit from $1$ to $0$, with the other bits unchanged. This happens with probability
\begin{align*}
\frac{1}{\mu} \left(1-\frac1n\right)^{n-1} \frac1n\ge \frac{1}{e\mu n}.
\end{align*}
From Lemma~\ref{lem:sur}, we know that any Pareto front point will be maintained in all future generations once reached. Hence, $1^i0^{n-i} \in P_{\tau}$ for all $\tau \ge t$. Thus, the expected number of iterations to reach $1^{i-1}0^{n-i+1}$ is at most $e\mu n$. Since we have at most $n$ such Pareto optima to reach, we know that the expected number of iterations to cover all remaining Pareto front points is at most $e\mu n^2$.

In summary, the full Pareto front will be covered in at most $2e\mu n^2$ iterations in expectation.
\end{proof}


\section{Conclusion}\label{sec:con}

Motivated by the observed difficulty of the \NSGA for many objectives, this paper resorted to the \SMS, a variant of the steady-state \NSGA, and proved that, different from the \NSGA, it efficiently solves the \mojzj problem. Noting that the \SMS also employs non-dominated sorting, but replaces the crowding distance with the hypervolume, this result together with the ones of \cite{WiethegerD23,OprisDNS24} for the \NSGAthree supports our conclusion that non-dominated sorting is a good building block for MOEAs, but the crowding distance has deficiencies for more than two objectives.

We also showed that the stochastic population update proposed in~\cite{BianZLQ23} for the bi-objective \SMS becomes less effective for many objectives. All these results in a very rigorous manner support the general knowledge that multi-objective optimization becomes increasingly harder with growing numbers of objectives. 

On the positive side, we showed that the advantages of heavy-tailed mutation, previously observed in single- and bi-objective optimization, remain unchanged when increasing the number of objectives.

Given that our first result shows advantages of the \SMS and there is only one previous work performing rigorous runtime analyses for this algorithm, we extended our knowledge in this direction by proving competitive runtime guarantees for this algorithm on the two most prominent bi-objective benchmarks. They show that the \SMS performs here at least at well as the \NSGA.

\subsection*{Subsequent Work}

In~\cite{WiethegerD24}, the runtime of several MOEAs on many-objective problems was analyzed. Concerning the runtime of the classic \SMS on the \mojzj benchmark, our guarantee of $O(\mu M n^k)$ was improved to $O(\mu m n^k)$. To prove this result, our general technical lemma for the survival of individuals in the population (Lemma~\ref{lem:sur}) was used. The improvement then stemmed from detecting a faster way how all the MOEAs regarded in that work discover the Pareto front. The work~\cite{WiethegerD24} did not regard the \SMS with stochastic population update and with heavy-tailed mutation. Since both shorten the time to find an additional extremal block (see the proof of Lemma~\ref{lem:full}), it is clear that analogous improvements are possible for the \SMS with stochastic population update and heavy-tailed mutation, giving the bounds $O(\mu m n^k \min\{1, \mu / 2^k\})$ and $O(\mu m k^{\beta-0.5-k} n^k)$ respectively. 

\subsection*{Acknowledgments}
This work was supported by National Natural Science Foundation of China (Grant No. 62306086), Science, Technology and Innovation Commission of Shenzhen Municipality (Grant No. GXWD20220818191018001), Guangdong Basic and Applied Basic Research Foundation (Grant No. 2019A1515110177).
This research benefited from the support of the FMJH Program Gaspard Monge for optimization and operations research and their interactions with data science.

\newcommand{\etalchar}[1]{$^{#1}$}

}
\end{document}